\documentclass{llncs}

\usepackage[utf8]{inputenc}

\usepackage{amsmath}
\usepackage{amsfonts}
\usepackage{amssymb}

\usepackage[round]{natbib}
\usepackage{color, soul}
\usepackage[dvipsnames]{xcolor}
\usepackage[colorlinks=true,linkcolor=blue,citecolor=blue,pdfborder={0 0 0}]{hyperref}
\usepackage[capitalise]{cleveref}
\definecolor{brown}{rgb}{0.59, 0.29, 0.0}

\usepackage{dsfont}
\usepackage{bm}
\usepackage[normalem]{ulem}
\usepackage{graphicx}
\usepackage{enumerate}
\usepackage{paralist}
\usepackage{subcaption}

\usepackage{bbm}
\usepackage{booktabs}
\usepackage{verbatim}
\usepackage{algorithm}
\usepackage{algorithmic}
\usepackage{multirow}
\def\argmax{\mathop{\rm arg\, max}}
\def\argmin{\mathop{\rm arg\, min}}

\begin{document}

%

%

\title{Max $K$-armed bandit: \\ On the ExtremeHunter algorithm and beyond}

\author{Mastane Achab\inst{1} \and Stephan Clémençon\inst{1} \and Aurélien Garivier\inst{2} \and Anne Sabourin\inst{1} \and Claire Vernade\inst{1}}
\institute{LTCI, Télécom ParisTech, Université Paris-Saclay \and IMT, Université de Toulouse}

\maketitle

\begin{abstract}
This paper is devoted to the study of the \textit{max K-armed bandit problem}, which consists in sequentially allocating resources in order to detect extreme values. Our contribution is twofold.
We first significantly refine the analysis of the \textsc{ExtremeHunter} algorithm carried out in
\citet{carpentier2014extreme}, and next propose an alternative approach, showing that, remarkably, Extreme Bandits can be reduced to a classical version of the bandit problem to a certain extent. Beyond the formal analysis, these two approaches are compared through numerical experiments.
\end{abstract}


\section{Introduction}

In a classical multi-armed bandit (MAB in abbreviated form) problem, the objective is to find a  strategy/policy in order to
sequentially explore and exploit $K$ sources of gain, referred to as \textit{arms}, so as to
maximize the expected cumulative gain.
Each arm $k\in\{1,\; \ldots,\; K  \}$ is characterized by an unknown probability distribution $\nu_k$.
At each round $t\ge 1$, a strategy $\pi$ picks an arm $I_t=\pi((I_1,
X_{I_1,1}), \; \ldots,\; (I_{t-1}, X_{I_{t-1}, t-1}))$ and receives a random reward
$X_{I_t, t}$
sampled from distribution $\nu_{I_t}$.
Whereas usual strategies aim at finding and exploiting the arm with highest expectation, the quantity of interest in many applications such as medicine, insurance or finance may not be the sum of the rewards, but rather the \emph{extreme} observations (even if it might mean replacing loss minimization by gain maximization in the formulation of the practical problem).
In such situations, classical bandit algorithms can be significantly sub-optimal: the "best" arm should not be defined as that with highest expectation, but as that
producing the maximal values. This setting, referred to as \textit{extreme bandits} in \cite{carpentier2014extreme},
was originally introduced by \cite{cicirello2005aaai} by the name of
\textit{max $K$-armed bandit problem}.
In this framework, the goal pursued is to obtain the highest possible
reward during the first $n\ge 1$ steps. For a given arm $k$, we denote by
\[
G_n^{(k)}=\max_{1\le t\le n}X_{k, t}
\]
the maximal value taken until round $n\ge 1$ and assume that, in expectation,  there is a unique optimal arm
\[
k^\ast= \argmax_{1\le k\le K} \mathbb{E}[G_n^{(k)}]\ .
\]
The expected \emph{regret} of a strategy $\pi$ is here defined as
\begin{equation}
\mathbb{E}[R_n]= \mathbb{E}[G_n^{(k^\ast)}] - \mathbb{E}[G_n^{(\pi)}],
\label{eq:def_regret}
\end{equation}
where $G_n^{(\pi)}=\max_{1\le t\le n}X_{I_t, t}$ is the maximal value
observed when implementing strategy $\pi$.
When the supports of the reward distributions (\textit{i.e.} the $\nu_k$'s) are bounded, no-regret is expected provided that every arm can be sufficiently explored, refer to \cite{nishihara2016no} (see
also \cite{david2016pac} for a PAC approach). If infinitely many
arms are possibly involved in the learning strategy, the challenge is then to explore and exploit optimally the unknown reservoir
of arms, see~\cite{carpentier2015simple}.
When the rewards are unbounded in contrast, the situation is quite
different: the best arm is that for which the maximum $G_n^{(k)}$ tends to
infinity faster than the others.
In \cite{nishihara2016no}, it is shown that, for unbounded distributions, no policy can achieve no-regret without restrictive assumptions on the distributions.
In accordance with the literature, we focus on a classical framework in
extreme value
analysis. Namely, we assume that the reward distributions are  \emph{heavy-tailed}.
Such Pareto-like laws are widely used to model extremes in many applications, where a conservative approach to risk assessment might be relevant (\textit{e.g.} finance, environmental risks).
Like in \cite{carpentier2014extreme}, rewards are assumed to be distributed as second order Pareto laws in the present article. For the sake of completeness, we recall that
a probability law with cdf $F(x)$ belongs to the
$(\alpha, \beta, C, C')$-second order Pareto family if, for every $x\ge
0$,
\begin{equation}
\label{eq:def_pareto}
|1-Cx^{-\alpha}-F(x)|\le C'x^{-\alpha(1+\beta)}\, ,
\end{equation}
where $\alpha, \beta, C \text{ and } C'$ are strictly positive constants, see \textit{e.g.} \cite{resnick2007heavy}.
In this context, \citet{carpentier2014extreme} have proposed the
\textsc{ExtremeHunter} algorithm to solve the \textit{extreme bandit} problem and provided a regret analysis.

The contribution of this paper is twofold. First, the regret analysis of the \textsc{ExtremeHunter}
algorithm is significantly improved, in a nearly optimal fashion.
This essentially relies on a new technical result of independent interest
(see Theorem~\ref{thm1} below), which provides a bound for the difference between the expectation of the maximum among independent realizations $X_1,\; \ldots,\; X_T$ of a  $(\alpha, \beta, C, C')$-second order Pareto distribution,
$\mathbb{E}[\max_{1\le i\le T} X_i]$ namely, and its rough
approximation $(TC)^{1/\alpha}\Gamma(1-1/\alpha)$.
As a by-product, we propose a more simple \textsc{Explore-Then-Commit} strategy
that offers the same theoretical guarantees as  \textsc{ExtremeHunter}.
Second, we explain how extreme bandit can be reduced to a classical bandit problem to a certain extent. We show that a regret-minimizing strategy such as
\textsc{Robust-UCB} (see~\cite{bubeck2013bandits}), applied on correctly
left-censored rewards, may also reach a very good performance. This claim is
supported by theoretical guarantees on the number of pulls of the best arm
$k^\ast$ and by numerical experiments both at the same time. From a practical angle, the main drawback of this alternative approach
consists in the fact that its implementation requires some knowledge of the complexity of the problem (\textit{i.e.} of the gap between the first-order Pareto coefficients of the first
and second arms). In regard to its theoretical analysis, efficiency is proved for large horizons only.

This paper is organized as follows.
Section~\ref{sec:control} presents the technical result mentioned above, which next permits to carry out a refined regret analysis of the \textsc{ExtremeHunter} algorithm in \cref{sec:upper_bound}.
In \cref{sec3}, the regret bound thus obtained is proved to be nearly optimal: precisely, we establish a lower bound under the assumption that the distributions are close
enough to Pareto distributions showing the regret bound is sharp in this situation. In \cref{sec4}, reduction of the extreme bandit problem to
a classical bandit problem is explained at length, and an algorithm resulting from this original view is then described. Finally, we provide a preliminary
numerical study that permits to compare the two approaches from an experimental perspective. Due to space limitations, certain technical proofs are deferred to the Supplementary Material.


\section{Second-order Pareto distributions: approximation of the expected maximum among i.i.d. realizations
		 }\label{sec:control}

In the extreme bandit problem, the key to controlling the behavior of
explore-exploit strategies is to approximate the expected payoff of a
fixed arm $k\in\{1,\; \ldots,\; K  \}$. The main result of this section, stated in Theorem~\ref{thm1}, provides such control: it significantly improves upon the result originaly obtained by
\cite{carpentier2014extreme} (see Theorem 1 therein). As shall be next shown in Section~\ref{sec:upper_bound}, this refinement has substantial consequences on the regret bound.

In \cite{carpentier2014extreme}, the distance between the
expected maximum of independent realizations of a $(\alpha,\beta,
C,C')$-second order Pareto and the corresponding expectation of a Fr\'echet
distribution $(TC)^{1/\alpha}\Gamma(1-1/\alpha)$ is controlled as follows:
\begin{multline*}
\left|\mathbb{E}\left[\max_{1\le i\le
	T}X_i\right]-(TC)^{1/\alpha}\Gamma(1-1/\alpha)\right|
\le
\frac{4D_2C^{1/\alpha}}{T^{1-1/\alpha}}+\frac{2C'D_{\beta+1}}
{C^{\beta+1-1/\alpha}T^{\beta-1/\alpha}}\\+(2C'T)^\frac{1}{(1+\beta)\alpha}\ .
\end{multline*}
Notice that the leading term of this bound is $(2C'T)^{1/((1+\beta)\alpha)}$ as $T\rightarrow +\infty$. Below, we state a sharper result where, remarkably, this (exploding) term disappears, the contribution of the related component in the approximation error decomposition being proved as (asymptotically) negligible in contrast.
\begin{theorem}{\sc (Fr\'echet approximation bound)}
	\label{thm1}
	If  $X_1,\; \ldots,\; X_T$ are i.i.d. r.v.'s drawn from
	a $(\alpha, \beta, C, C')$-second order Pareto distribution with $\alpha>1$ and $T\ge Q_1$, where $Q_1$ is the constant depending only on $\alpha, \beta, C$ and $C'$ given in \cref{Q1} below, then,
	\begin{equation*}
		\begin{split}
	&\left|\mathbb{E}\left[\max_{1\le i\le
		T}X_i\right]-(TC)^{1/\alpha}\Gamma(1-1/\alpha)\right|\\
	&\le
	\frac{4D_2C^{1/\alpha}}{T^{1-1/\alpha}}+\frac{2C'D_{\beta+1}}{C^{\beta+1-1/\alpha}T^{\beta-1/\alpha}}
	+2(2C'T)^\frac{1}{(1+\beta)\alpha}e^{-HT^\frac{\beta}{\beta+1}}\\
	&= \underset{T\to\infty}{o} (T^{1/\alpha}),
		\end{split}
\end{equation*}
	where $H=C(2C')^{1/(\alpha(1+\beta))}/2$.
	In particular, if $\beta \ge 1$, we have:
	\[
	\left|\mathbb{E}\left[\max_{1\le i\le
		T}X_i\right]-(TC)^{1/\alpha}\Gamma(1-1/\alpha)\right| = o(1) \text{ as } T\rightarrow +\infty.
	\]
\end{theorem}
	We emphasize that the bound above shows that the distance of
	$\mathbb{E}[\max_{1\le i\le	T}X_i]$ to the Fr\'echet mean $(TC)^{1/\alpha}\Gamma(1-\frac{1}{\alpha})$ actually vanishes as $T\to \infty$ as soon as $\beta \ge 1$, a property that shall be useful in Section~\ref{sec:upper_bound} to study the behavior
	of learning algorithms in the extreme bandit setting.
\begin{proof}
	Assume that $T\ge Q_1$, where
	\begin{equation}
		\label{Q1}
	Q_1 = \frac{1}{2C'}\max\left\{(2C'/C)^{(1+\beta)/\beta}, \; (8C)^{1+\beta}\right\}.
	\end{equation}
	As in the proof of Theorem 1 in \cite{carpentier2014extreme}, we consider the quantity $B=(2C'T)^{1/((1+\beta)\alpha)}$ that serves as a
	cut-off between tail and bulk behaviors. Observe that
	\begin{multline*}
	\left| \mathbb{E}\left[\max_{1\le i\le T}X_i\right] -
	(TC)^{1/\alpha}\Gamma(1-1/\alpha) \right|
	\le \\
	 \left| \int_{0}^\infty\left\{ 1-\mathbb{P}\left(\max_{1\le i\le
	T}X_i\le x\right)-1+e^{-TCx^{-\alpha}}\right\}\mathrm{d}x
	\right| \\
	\le \left|\int_{0}^B\left\{ \mathbb{P}\left(\max_{1\le i\le T}X_i\le x\right)
	-e^{-TCx^{-\alpha}} \right\} \mathrm{d}x \right|\\
	+ \left|\int_{B}^\infty\left\{ \mathbb{P}\left(\max_{1\le i\le T}X_i\le
	x\right)-e^{-TCx^{-\alpha}}\right\}\mathrm{d}x\right| .
	\end{multline*}
	For $p\in\{2, \beta+1\}$, we set $D_p=\Gamma(p-\frac{1}{\alpha})/\alpha$. Equipped with this notation, we may write
	\begin{equation*}
	\left|\int_{B}^\infty\left\{\mathbb{P}\left(\max_{1\le i\le T}X_i\le
	x\right)-e^{-TCx^{-\alpha}}\right\}\mathrm{d}x\right|
	\le
	\frac{4D_2C^{1/\alpha}}{T^{1-1/\alpha}}+\frac{2C'D_{\beta+1}}
	{C^{\beta+1-1/\alpha}T^{\beta-1/\alpha}}.
	\end{equation*}
	Instead of loosely bounding the bulk term by $B$, we write
	\begin{equation}
	\label{bulk}
	\left|\int_{0}^B\left\{\mathbb{P}\left(\max_{1\le i\le T}X_i\le
	x\right)-e^{-TCx^{-\alpha}}\right\}\mathrm{d}x\right|
	\le B\,\mathbb{P}\left(X_1\le B\right)^T +
	\int_{0}^Be^{-TCx^{-\alpha}}\mathrm{d}x\ .
	\end{equation}
	First, using \eqref{eq:def_pareto} and the inequality
	$C'B^{-(1+\beta)\alpha}\le CB^{-\alpha}/2$ (a direct consequence of \cref{Q1}), we obtain
	\begin{align*}
	&\mathbb{P}(X_1\le B)^T
	\le \left(1-CB^{-\alpha}+C'B^{-(1+\beta)\alpha}\right)^T\\
	&\le\left(1-\frac{1}{2}CB^{-\alpha}\right)^T
	\le e^{-\frac{1}{2}TCB^{-\alpha}}= e^{-HT^{\beta/(\beta+1)}}.
	\end{align*}
	Second, the integral in \cref{bulk} can be bounded as follows:
	\begin{equation*}
	\int_{0}^B e^{-TCx^{-\alpha}}\mathrm{d}x	\le Be^{-TCB^{-\alpha}}
	 = (2C'T)^{1/((1+\beta)\alpha)}e^{-2HT^{\beta/(\beta+1)}}.
\end{equation*}
This concludes the proof.
\end{proof}


\section{The \textsc{ExtremeHunter} and \textsc{ExtremeETC} algorithms}
\label{sec:upper_bound}

In this section, the tighter control provided by Theorem~\ref{thm1} is used in
order to refine the  analysis of the \textsc{ExtremeHunter}
algorithm (Algorithm~\ref{alg:EH}) carried out in  \cite{carpentier2014extreme}.
This theoretical analysis is also shown to be valid for \textsc{ExtemeETC}, a novel algorithm we next propose, that greatly improves upon \textsc{ExtremeHunter}, regarding computational efficiency.

\subsection{Further Notations and Preliminaries}
\label{notations}
Throughout the paper, the indicator function of any event $\mathcal{E}$ is denoted by $\mathbbm{1}\{\mathcal{E} \}$ and $\bar{\mathcal{E}}$ means the complementary event of $\mathcal{E}$.
We assume that the reward related to each arm $k\in\{1,\; \ldots,\; K  \}$ is drawn from a $(\alpha_k, \beta_k, C_k, C')$-second
order Pareto distribution.
Sorting the tail indices by increasing order of magnitude, we use the classical notation for order statistics:
$\alpha_{(1)}\le \dots\le\alpha_{(K)}$.
We assume that $\alpha_{(1)}>~1$, so that the random rewards have finite expectations, and suppose that the strict inequality $\alpha_{(1)}<\alpha_{(2)}$ holds true.
We also denote by $T_{k,t}$ the number of times the arm $k$ is pulled up to time $t$.
For $1\le k\le K$ and $i\ge 1$, the r.v. $\widetilde{X}_{k, i}$ is the reward obtained
at the $i$-th draw of arm $k$
if $i\le T_{k, n}$ or a new r.v. drawn from $\nu_k$ independent from the other r.v.'s otherwise.

We start with a preliminary lemma supporting the intuition that
the tail index $\alpha$ fully governs the extreme bandit problem. It will
allow to show next that the algorithm picks the right arm after the exploration phase, see Lemma~\ref{lemTast}.
\begin{lemma}{\sc (Optimal arm)}
	\label{lem:smallest_alpha}
	For $n$ larger than some constant $Q_4$ depending only on
        $(\alpha_k, \beta_k, C_k)_{1\le k\le K}$ and $C'$, the optimal arm
        for the extreme bandit problem is given by:
				\begin{equation}
					\label{eq:smallestAlpha}
            k^\ast =~\argmin_{1\le k\le K}\alpha_k
            = \argmax_{1\le k\le K} V_k,
				\end{equation}
				where $V_k = (nC_k)^{1/\alpha_k}\Gamma(1-1/\alpha_k)$.
\end{lemma}

\begin{proof}
	We first prove the first equality.
  It follows from \cref{thm1} that there exists a constant $Q_2$,
  depending only on $\{(\alpha_k, \beta_k, C_k)\}_{1\le k\le K}$ and
  $C'$, such that for any arm $k\in\{1,\; \ldots,\; K \}$,
  $|\mathbb{E}[G_n^{(k)}]-V_k |\le V_k/2$.  Then for $k\neq k^\ast$ we
  have, for all $n>Q_2$,
  $V_k/2 \le \mathbb{E}[G_n^{(k)}]\le \mathbb{E}[G_n^{(k^\ast)}]\le
  3V_{k^\ast}/2$.
	Recalling that $V_k$ is proportional to $n^{1/\alpha_k}$, it follows that
  $\alpha_{k^\ast}=\min_{1\le k\le K} \alpha_k$.
	Now consider the following quantity:
\begin{equation}
 	\label{def:Q3}
Q_3 = \max_{k\neq k^\ast}\left[\frac{2C_k^{1/\alpha_k}\Gamma(1-1/\alpha_k)}{C_{k^\ast}^{1/\alpha_{k^\ast}}\Gamma(1-1/\alpha_{k^\ast})}\right]^{1/(1/\alpha_{k^\ast}-1/\alpha_k)}\ .
\end{equation}
For $n>Q_4=\max(Q_2, Q_3)$, we have $V_{k^\ast}>2V_k$ for any suboptimal  arm $k\neq k^\ast$,
which proves the second equality.

\end{proof}

From now on, we assume that $n$ is large enough for  \cref{lem:smallest_alpha} to apply.

 \subsection{The \textsc{ExtremeHunter} algorithm \citep{carpentier2014extreme}}\label{sec:extremeHbackground}
  Before developing a novel analysis of the extreme bandit problem in Section~\ref{sec:extremeHbackground} (see Theorem~\ref{theo:regretExtremeHunter}), we recall  the main
 features of \textsc{ExtremeHunter}, and in particular the estimators and
 confidence intervals involved in the indices of this optimistic
  policy.

\begin{algorithm}[H]
	\caption{\textsc{ExtremeHunter} \citep{carpentier2014extreme}}
	\label{alg:EH}
	\begin{algorithmic}[1]
		\STATE {\bfseries Input:} $K$: number of arms, $n$: time horizon, $b>0$
 	 such that $b\le\min_{1\le k\le K}\beta_k$, $N$: minimum number of pulls
 	 of each arm (\cref{N}).
 	 \STATE {\bfseries Initialize:} Pull each arm $N$ times.
 	 \STATE {\bfseries for} $k=1, \dots, K$ {\bfseries do}
 		 \STATE \quad Compute estimators $\widehat{h}_{k, KN}=\widetilde{h}_k(N)$ (\cref{estimH}) and
 		 $\widehat{C}_{k, KN}=\widetilde{C}_k(N)$ (\cref{estimC})
 		 \STATE \quad Compute index $B_{k, KN}$ (\cref{indexEH})
 	 \STATE {\bfseries end for}
	 \STATE Pull arm $I_{KN+1} = \argmax_{1\le k\le K}B_{k, KN}$
 	 \STATE {\bfseries for} $t=KN+2, \dots, n$ {\bfseries do}
		 \STATE \quad Update estimators $\widehat{h}_{I_{t-1}, t-1}$ and
		 $\widehat{C}_{I_{t-1}, t-1}$
		 \STATE \quad Update index $B_{I_{t-1}, t-1}$
 		 \STATE \quad Pull arm $I_t = \argmax_{1\le k\le K}B_{k, t-1}$
 	 \STATE {\bfseries end for}
	\end{algorithmic}
\end{algorithm}

\cref{thm1} states that for any arm $k\in\{1,\; \ldots,\; K  \}$, $\mathbb{E}[G_n^{(k)}]\approx (C_k
n)^{1/{\alpha_k}}\Gamma(1-1/\alpha_k)$. Consequently, the
optimal strategy in hindsight always pulls the arm
$ k^\ast=\argmax_{1\le k\le K}
\{(nC_k)^{1/\alpha_k}\Gamma(1-1/\alpha_k)\}$.
At each round and for each arm $k\in\{1,\; \ldots,\; K  \}$, \textsc{ExtremeHunter} algorithm
\citep{carpentier2014extreme} estimates the coefficients $\alpha_k$ and $C_k$ (but not $\beta_k$, see Remark 2 in \citet{carpentier2014extreme}). The
corresponding confidence intervals are detailed below. Then, following the \textit{optimism-in-the-face-of-uncertainty} principle (see \citep{auer2002finite} and references therein), the strategy  plays the arm
maximizing an optimistic plug-in estimate of
$(C_kn)^{1/\alpha_k}\Gamma(1-1/\alpha_k)$. To that purpose,
Theorem 3.8 in
\citet{carpentier2014adaptive} and Theorem 2 in
\citet{carpentier2014honest} provide estimators $\widetilde{\alpha}_k(T)$ and
$\widetilde{C}_k(T)$ for $\alpha_k$ and $C_k$ respectively, after $T$ draws of arm $k$. Precisely, the estimate $\widetilde\alpha_k(T)$ is given by
\begin{equation*}
	\label{estimAlpha}
	\widetilde\alpha_k(T)=\log\left( \frac{ \sum_{t=1}^{T}\mathbbm{1}\{X_t>e^r\} }{ \sum_{t=1}^{T}\mathbbm{1}\{X_t>e^{r+1}\} } \right)\, ,
\end{equation*}
where $r$ is chosen in an adaptive fashion based on Lepski's method, see \citep{MR1091202}, while
the estimator of $C_k$ considered is
\begin{equation}
	\label{estimC}
\widetilde{C}_k(T)=T^{-2b/(2b+1)}\sum_{i=1}^T \mathbbm{1}\{\widetilde{X}_{k,i}\ge
T^{\widetilde{h}_k(T)/(2b+1)}\},
\end{equation}
where
\begin{equation}
	\label{estimH}
	\widetilde{h}_k(T)=\min(1/\widetilde{\alpha}_k(T), 1)\ .
\end{equation}
The authors also provide  finite sample error bounds for $T\ge
N$, where
\begin{equation}
  \label{N}
 N= A_0(\log n)^{2(2b+1)/b},
\end{equation}
with $b$ a known lower bound on the $\beta_k$'s ($b\le~\min_{1\le k\le K}\beta_k$), and
$A_0$ a constant depending only on $(\alpha_k, \beta_k, C_k)_{1\le k\le
K}$ and $C'$.
These error bounds naturally define confidence intervals  of respective widths $\Lambda_1$ and $ \Lambda_2$
at level $\delta_0$ defined by
\begin{equation}
	\label{delta0}
\delta_0 = n^{-\rho},\quad \text{where}\quad \rho=\frac{2\alpha_{k^\ast}}{\alpha_{k^\ast}-1}.
\end{equation}
More precisely, we have
\begin{equation}
\mathbb{P}\left(\left|\frac{1}{\alpha_k}-\widetilde{h}_k(T)\right| \le
\Lambda_1(T),\; \left|C_k-\widetilde{C}_k(T)\right|\le \Lambda_2(T) \right) \ge 1-2\delta_0,  \label{errorHC}
\end{equation}
where
\begin{equation*}
\Lambda_1(T)=D\sqrt{\log(1/\delta_0)}T^{-b/(2b+1)} \text{ and }
\Lambda_2(T)=E\sqrt{\log(T/\delta_0)}\log(T)T^{-b/(2b+1)},
\end{equation*}
denoting by $D$ and $E$ some constants depending only on $(\alpha_k, \beta_k,
C_k)_{1\le k\le K}$ and $C'$.
When $T_{k, t}\ge N$, denote by $\widehat{h}_{k, t}=\widetilde{h}_k(T_{k, t})$ and
$\widehat{C}_{k, t}=\widetilde{C}_k(T_{k, t})$ the estimators based on the
$T_{k, t}$ observations for simplicity.
\textsc{ExtremeHunter}'s index $B_{k, t}$ for arm $k$ at time $t$, the optimistic
proxy for $\mathbb{E}[G_n^{(k)}]$, can be then written as
\begin{equation}
\label{indexEH}
B_{k, t}=\widetilde{\Gamma}\left(1-\widehat{h}_{k, t}-\Lambda_1(T_{k, t})\right)
\left(\left(\widehat{C}_{k, t}+\Lambda_2(T_{k, t})\right)n\right)^{\widehat{h}_{k, t}+\Lambda_1(T_{k,
t})}\, ,
\end{equation}
where $\widetilde{\Gamma}(x)=\Gamma(x)$ if $x>0$ and $+\infty$ otherwise.
\medskip

\noindent {\bf On computational complexity.}
Notice that after the initialization phase, at each time $t>KN$, \textsc{ExtremeHunter} computes estimators $\widehat{h}_{I_t, t}$
and $\widehat{C}_{I_t, t}$, each having a time complexity linear with the number of samples $T_{I_t, t}$ pulled from arm $I_t$ up to time $t$.
Summing on the rounds reveals that \textsc{ExtremeHunter}'s time complexity is quadratic with the time horizon $n$.

\subsection{\textsc{ExtremeETC}: a computationally appealing alternative}\label{sec:extremeETC}
In order to reduce the restrictive time complexity discussed previously, we now propose the \textsc{ExtremeETC} algorithm, an \textit{Explore-Then-Commit} version of \textsc{ExtremeHunter},
which offers similar theoretical guarantees.

\begin{algorithm}[H]
 \caption{\textsc{ExtremeETC}}
 \label{alg:ETC}
 \begin{algorithmic}[1]
	 \STATE {\bfseries Input:} $K$: number of arms, $n$: time horizon, $b>0$
	 such that $b\le\min_{1\le k\le K}\beta_k$, $N$: minimum number of pulls
	 of each arm (\cref{N}).
	 \STATE {\bfseries Initialize:} Pull each arm $N$ times.
	 \STATE {\bfseries for} $k=1, \dots, K$ {\bfseries do}
		 \STATE \quad Compute estimators $\widehat{h}_{k, KN}=\widetilde{h}_k(N)$ (\cref{estimH}) and
 		 $\widehat{C}_{k, KN}=\widetilde{C}_k(N)$ (\cref{estimC})
		 \STATE \quad Compute index $B_{k, KN}$ (\cref{indexEH})
	 \STATE {\bfseries end for}
	 \STATE Set $I_\text{winner} = \argmax_{1\le k\le K}B_{k, KN}$
	 \STATE {\bfseries for} $t=KN+1, \dots, n$ {\bfseries do}
		 \STATE \quad Pull arm $I_\text{winner}$
	 \STATE {\bfseries end for}
 \end{algorithmic}
\end{algorithm}

After the initialization phase, the \textit{winner arm}, which has maximal index $B_{k, KN}$, is fixed
and is pulled in all remaining rounds. Then \textsc{ExtremeETC}'s time complexity, due to the computation of $\widehat{h}_{k, KN}$ and
$\widehat{C}_{k, KN}$ only, is $\mathcal{O}\left(KN\right)=\mathcal{O}\left((\log n)^{2(2b+1)/b}\right)$, which is considerably
faster than quadratic time achieved by \textsc{ExtremeHunter}. For clarity, Table~\ref{complexities} summarizes time and memory complexities of both algorithms.

\begin{table}
	\begin{center}
	\begin{tabular}{l c c}
	  \toprule
	  \textbf{Complexity} & \textbf{\textsc{ExtremeETC}} & \textbf{\textsc{ExtremeHunter}}\\
	  \midrule
	  Time & $\mathcal{O}\bigl((\log n)^\frac{2(2b+1)}{b}\bigr)$ & $\mathcal{O}(n^2)$ \\
	  Memory & $\mathcal{O}\bigl((\log n)^\frac{2(2b+1)}{b}\bigr)$ & $\mathcal{O}(n)$ \\
	  \bottomrule
	  \end{tabular}
	  \bigskip

		\caption{Time and memory complexities required for estimating $(\alpha_k, C_k)_{1\le k\le K}$ in \textsc{ExtremeETC} and \textsc{ExtremeHunter}.}
		\label{complexities}
	\end{center}
\end{table}

Due to the significant gain of computational time, we used the \textsc{ExtremeETC} algorithm in our simulation study (Section~\ref{experiments}) rather than \textsc{ExtremeHunter}.
\medskip

\noindent {\bf Controlling the number of suboptimal rounds.}
We introduce a high probability event that corresponds to the favorable
situation where, at each round, all coefficients $(1/\alpha_k, C_k)_{1\le k \le
K}$ simultaneously belong to the confidence intervals recalled in the previous subsection.

\begin{definition}
	\label{event_xi}
	The event $\xi_1$ is the event on which the bounds
	\begin{equation*}
	\left|\frac{1}{\alpha_k}-\widetilde{h}_k(T)\right|\le \Lambda_1(T)\quad
	\mbox{and}\quad
	\left|C_k-\widetilde{C}_k(T)\right|\le \Lambda_2(T)
	\end{equation*}
hold true for any $1\le k\le K$ and $N\le T\le n$.
\end{definition}

The union bound combined with \eqref{errorHC} yields
\begin{equation}
\label{proba_xi1}
\mathbb{P}(\xi_1)\ge 1-2Kn\delta_0.
\end{equation}

\begin{lemma}
	\label{lemTast}
	For $n>Q_5$, where $Q_5$ is the constant defined in \eqref{def:Q5}, \textsc{ExtremeETC} and \textsc{ExtremeHunter}
	always pull the optimal arm after the initialization phase on the event $\xi_1$.
	Hence, for any suboptimal arm $k\neq k^\ast$, we have on $\xi_1$:
	\begin{equation*}
	T_{k, n} = N \quad \text{and thus}\quad T_{k^\ast, n} = n - (K-1)N.
	\end{equation*}
\end{lemma}

\begin{proof}
	Here we place ourselves on the event $\xi_1$.
	For any arm $1\le k\le K$, Lemma 1 in \citet{carpentier2014extreme} provides lower
	and upper bounds for $B_{k, t}$ when $T_{k, t}\ge N$
	\begin{equation}
		\label{boundsBk}
	V_k
	\le B_{k, t} 
	\le V_k 
 \left(1+F\log n \sqrt{\log(n/\delta_0)}T_{k, t}^{-b/(2b+1)}\right)\, ,
	\end{equation}
	where $F$ is a constant which depends only on $(\alpha_k, \beta_k,
	C_k)_{1\le k\le K}$ and C'.
 Introduce the horizon $Q_5$, which depends on $(\alpha_k, \beta_k,
 C_k)_{1\le k\le K}$ and C'
\begin{equation}
\label{def:Q5}
Q_5 = \max\left(e^{\left(F\sqrt{1+\rho}A_0^{-b/(2b+1)}\right)^2}, Q_4\right).
\end{equation}
Then the following \cref{lem:Q5}, proved in \cref{appendix}, tells us that for $n$ large enough, the exploration made during the initialization phase is enough to find the optimal arm, with high probability.
\begin{lemma}
  \label{lem:Q5}
If $n>Q_5$, we have under the event $\xi_1$ that for any suboptimal arm $k\neq k^\ast$ and any time $t>KN$
that $B_{k, t}<B_{k^\ast, t}$\ .
\end{lemma}
	Hence the
	optimal arm is pulled at any time $t> KN$.
\end{proof}

The following result immediately follows from \cref{lemTast}.

\begin{corollary}
	\label{cor_T}
	For $n$ larger than some constant depending only on $(\alpha_k, \beta_k, C_k)_{1\le k\le K}$ and $C'$ we have under $\xi_1$
	\begin{equation*}
	T_{k^\ast, n}\ge n/2.
	\end{equation*}
\end{corollary}

\noindent {\bf Upper bounding the expected extreme regret.}
The upper bound on the expected
extreme regret stated in the theorem below improves upon that given in \citet{carpentier2014extreme} for \textsc{ExtremeHunter}. It is also valid for \textsc{ExtremeETC}.

\begin{theorem}\label{theo:regretExtremeHunter}
	\label{upper}
	For \textsc{ExtremeETC} and \textsc{ExtremeHunter}, the expected
	extreme regret is upper bounded as follows
	\begin{equation*}
	\mathbb{E}[R_n]
	= \mathcal{O}\left((\log n)^{2(2b+1)/b}n^{-(1-1/\alpha_{k^\ast})}+n^{-(b-1/\alpha_{k^\ast})}\right),
	\end{equation*}
as $n\rightarrow +\infty$. If $b\ge 1$, we have in particular	$\mathbb{E}[R_n]= o(1)$ as $n\rightarrow +\infty$.
\end{theorem}

The proof  of \cref{upper} is deferred to  \cref{appendix}.  It closely follows that of Theorem 2 in \cite{carpentier2014extreme}, the main difference being that their concentration bound (Theorem 1 therein) can be replaced by our tighter bound (see Theorem~\ref{thm1} in the present paper).
Recall that in Theorem 2 in \citet{carpentier2014extreme}, the upper bound on the expected
extreme regret for \textsc{ExtremeHunter} goes to infinity when $n\rightarrow +\infty$:
\begin{equation}
\label{upper_carp_valko}
\mathbb{E}[R_n]= \mathcal{O}\left(n^\frac{1}{(1+b)\alpha_{k^\ast}}\right).
\end{equation}

In contrast, in \cref{upper} when $b\ge 1$, the upper bound obtained vanishes  when $n\rightarrow +\infty$.
In the case $b<1$, the upper bound still improves upon \cref{upper_carp_valko} by a factor
$n^{(\alpha_{k^\ast} b(b+1)-b)/((b+1)\alpha_{k^\ast})}>n^{b^2/(2\alpha_{k^\ast})}$.




\section{Lower bound on the expected extreme regret}
\label{sec3}

In this section we prove a lower bound on the expected extreme regret for
\textsc{ExtremeETC} and \textsc{ExtremeHunter} in specific cases.
We assume now that $\alpha_{(2)}>2\alpha_{k^\ast}^2/(\alpha_{k^\ast}-1)$ and we start with a
 preliminary result on second order Pareto distributions, proved in \cref{appendix}.

\begin{lemma}
	\label{power}
	If $X$ is a r.v. drawn from a $(\alpha, \beta, C, C')$-second
	order Pareto distribution and $r$ is a strictly positive constant,
	the distribution of the r.v. $X^r$ is a $(\alpha/r, \beta, C, C')$-second order Pareto.
\end{lemma}

In order to prove the lower bound on the expected extreme regret, we first establish that the event corresponding to the situation where the highest reward obtained by \textsc{ExtremeETC} and \textsc{ExtremeHunter} comes from
the optimal arm $k^\ast$ occurs with overwhelming probability. Precisely,
we denote by $\xi_2$ the event such that the bound
\begin{equation*}
\max_{k\neq k^\ast}\max_{1\le i\le N}\widetilde{X}_{k,i} \le \max_{1\le i\le
n-(K-1)N}\widetilde{X}_{k^\ast,i}.
\end{equation*}
holds true. The following lemma, proved in \cref{appendix}, provides a control of its probability of occurence.
\begin{lemma}
	\label{xi12}
	For $n$ larger than some constant depending only on $(\alpha_k, \beta_k, C_k)_{1\le k\le K}$ and $C'$, the following assertions hold true.
	\begin{enumerate}[(i)]
		\item  We have:
		\begin{equation*}
		\mathbb{P}(\xi_2)\ge 1-K\delta_0,
              \end{equation*}
              where $\delta_0$ is given in \cref{delta0}.
		\item Under the event $\xi_0=\xi_1\cap\xi_2$, the maximum reward obtained by
		\textsc{ExtremeETC} and \textsc{ExtremeHunter} comes from the optimal arm:
		\begin{equation*}
		\max_{1\le t\le n}X_{I_t, t} = \max_{1\le i\le
		n-(K-1)N}\widetilde{X}_{k^\ast, i} .
		\end{equation*}
	\end{enumerate}
\end{lemma}


The following  lower bound shows that the upper bound  (\cref{upper})  is actually tight in
the case $b\ge 1$.
\begin{theorem}
	\label{thm_lower}
	If $b\ge 1$ and $\alpha_{(2)}>2\alpha_{k^\ast}^2/(\alpha_{k^\ast}-1)$,
	the expected extreme regret of \textsc{ExtremeETC} and \textsc{ExtremeHunter} are lower bounded as
	follows
	\begin{equation*}
	\mathbb{E}[R_n] = \Omega\left((\log
	n)^{2(2b+1)/b}n^{-(1-1/\alpha_{k^\ast})}\right)\ .
	\end{equation*}
\end{theorem}

\begin{proof}
  Here, $\pi$ refers to either \textsc{ExtremeETC} or else \textsc{ExtremeHunter}. In order to bound from  below
  $\mathbb{E}[R_n]=\mathbb{E}[G_n^{(k^\ast)}]-\mathbb{E}[G_n^{(\pi)}]$, we start with bounding $\mathbb{E}[G_n^{(\pi)}]$ as follows

	\begin{align}
	\mathbb{E}\left[G_n^{(\pi)}\right]
	&= \mathbb{E}\left[\max_{1\le t\le n}X_{I_t, t}\right]
	= \mathbb{E}\left[\max_{1\le t\le n}X_{I_t,
	t}\mathbbm{1}\{\xi_0\}\right]+\mathbb{E}\left[\max_{1\le t\le n}X_{I_t,
	t}\mathbbm{1}\{\bar{\xi}_0 \}\right]\nonumber\\
	&\le \mathbb{P}(\xi_0)\mathbb{E}\left[\max_{1\le t\le n}X_{I_t, t}\ \Big|\
	\xi_0\right]
	+\sum_{k=1}^K\mathbb{E}\left[\max_{1\le i\le T_{k, n}}\widetilde{X}_{k,
	i}\mathbbm{1}\{\bar{\xi}_0 \}\right]\, ,\label{sum_xi12}
	\end{align}
	where  $\widetilde{X}_{k, i}$ has been  defined in \cref{notations}.
	From $(ii)$ in \cref{xi12}, we have
	\begin{equation}
	\label{maxisast}
	\mathbb{E}\left[\max_{1\le t\le n}X_{I_t, t}\ \Big|\ \xi_0\right]
	= \mathbb{E}\left[\max_{1\le i\le n-(K-1)N}\widetilde{X}_{k^\ast, i}\ \Big|\
	\xi_0\right]\ .
	\end{equation}
	In addition, in the sum of expectations
        on the right-hand-side of \cref{sum_xi12}, 
        $T_{k,n}$ may be roughly bounded from above
        by $n$. A straightforward application of  Hölder inequality yields
	\begin{equation}
	\label{holder}
	\sum_{k=1}^K\mathbb{E}\left[\max_{1\le i\le T_{k, n}}\widetilde{X}_{k,
	i}\mathbbm{1}\{\bar{\xi}_0 \}\right]
	\le \sum_{k=1}^K\left(\mathbb{E}\left[\max_{1\le i\le n}\widetilde{X}_{k,
	i}^\frac{\alpha_{k^\ast}+1}{2}\right]\right)^\frac{2}{\alpha_{k^\ast}+1}
	\mathbb{P}\left(\bar{\xi}_0 \right)^\frac{\alpha_{k^\ast}-1}{\alpha_{k^\ast}+1}\ .
	\end{equation}

	From $(i)$ in \cref{xi12} and \cref{proba_xi1}, we have
	$\mathbb{P}(\bar{\xi}_0 )\le K(2n+1)\delta_0$.
	By virtue of \cref{power}, the r.v. $\widetilde{X}_{k, i}^{(\alpha_{k^\ast}+1)/2}$ follows a
	$(2\alpha_k/(\alpha_{k^\ast}+1), \beta_k, C_k, C')$-second order Pareto
	distribution.
	Then, applying \cref{thm1} to the right-hand side of \eqref{holder} and
	using the identity \eqref{maxisast}, the upper bound \eqref{sum_xi12} becomes

	\begin{align}
	&\mathbb{E}\left[G_n^{(\pi)}\right]
	\le \mathbb{E}\left[\max_{1\le i\le n-(K-1)N}\widetilde{X}_{k^\ast, i} \mathbbm{1}\{\xi_0\}\right]\nonumber \\
	&+ \sum_{k=1}^K\left((n C_k)^\frac{\alpha_{k^\ast}+1}{2\alpha_k}\Gamma\left(1-\frac{\alpha_{k^\ast}+1}{2\alpha_k}\right)
	+o\left(n^\frac{\alpha_{k^\ast}+1}{2\alpha_k}\right)\right)^\frac{2}{\alpha_{k^\ast}+1}(K(2n+1)\delta_0)^\frac{\alpha_{k^\ast}-1}{\alpha_{k^\ast}+1}\nonumber \\
	&\le \mathbb{E}\left[\max_{1\le i\le n-(K-1)N}\widetilde{X}_{k^\ast, i}\right] +
	\mathcal{O}\left(n^{-(1-1/\alpha_{k^\ast})} \right), 	\label{GnEH}
	\end{align}
	where the last inequality comes from the definition of $\delta_0$.
	Combining \cref{thm1} and  \eqref{GnEH} we finally obtain the desired lower bound 
	\begin{equation*}
	\begin{split}
	&\mathbb{E}[R_n]
	= \mathbb{E}\left[G_n^{(k^\ast)}\right]-\mathbb{E}\left[G_n^{(\pi)}\right]\\
	&\ge
	\Gamma(1-1/\alpha_{k^\ast})C_{k^\ast}^{1/\alpha_{k^\ast}}\left(n^{1/\alpha_{k^\ast}}-(n-(K-1)N)^{1/\alpha_{k^\ast}}\right)
	+\mathcal{O}\left(n^{-(1-1/\alpha_{k^\ast})}\right)\\
	&= \frac{\Gamma(1-1/\alpha_{k^\ast})C_{k^\ast}^{1/\alpha_{k^\ast}}}{\alpha_{k^\ast}}(K-1)Nn^{-(1-1/\alpha_{k^\ast})}
	+\mathcal{O}\left(n^{-(1-1/\alpha_{k^\ast})}\right),
	\end{split}
	\end{equation*}
  where we used a Taylor expansion of $x\mapsto (1+x)^{1/\alpha_{k^\ast}}$ at zero for the last equality.
\end{proof}


\section{A reduction to classical bandits}
\label{sec4}

The goal of this section is to render explicit the connections between the max $K$-armed bandit considered in the present paper and a particular instance of the classical Multi-Armed Bandit (MAB) problem.

\subsection{MAB setting for extreme rewards}\label{sec:truncRewards}

In a situation where only the large rewards matter, an alternative to
the max $k$-armed problem would be to consider the expected cumulative
sum of the most `extreme' rewards, that is, those which exceeds a
given high threshold
$u$. 
For $k\in\{1,\; \ldots,\; K  \}$ and $t\in\{1,\; \ldots,\; n  \}$, we denote by $Y_{k, t}$ these new rewards
\begin{equation*}
Y_{k, t}=X_{k, t}\mathbbm{1}\{X_{k, t}>u\}\ .
\end{equation*}
In this context, the classical MAB problem  consists in maximizing the expected cumulative gain
\[
\mathbb{E}\left[G^{\text{MAB}}\right] = \mathbb{E}\left[ \sum_{t=1}^n Y_{I_t,t}\right].
\]

It turns out that for a high enough threshold $u$,   the unique optimal arm
for this MAB problem, $\argmax_{1\le k\le K}\mathbb{E}[Y_{k, 1}]$, is also the
optimal arm $k^\ast$ for the max $k$-armed  problem.
We still assume second order Pareto distributions for the random variables
$X_{k, t}$ and that all the hypothesis listed in Section~\ref{notations} hold true. 
The rewards $\{Y_{k, t}\}_{1\le k\le K, 1\le t\le T}$ are also heavy-tailed so that it is legitimate to 
attack this MAB problem with the  \textsc{Robust UCB} algorithm
\citep{bubeck2013bandits},
which assumes that the rewards have finite moments of order $1+\epsilon$
\begin{equation}
\label{eps_v}
\max_{1\le k\le K}\mathbb{E}\left[\left|Y_{k, 1}\right|^{1+\epsilon}\right]\le v\, ,
\end{equation}
where $\epsilon\in(0, 1]$ and $v>0$ are known constants. Given our second order
Pareto assumptions, it follows that \cref{eps_v} holds with
$1+\epsilon<\alpha_{(1)}$.
Even if the knowledge of  such constants $\epsilon$ and $v$ is a strong assumption, it is
still fair to compare \textsc{Robust UCB} to \textsc{ExtremeETC/Hunter}, which also
has strong requirements.
Indeed, \textsc{ExtremeETC/Hunter} assumes that $b$ and $n$
are known and verify conditions depending on unknown problem parameters (e.g. $n\ge Q_1$, see \cref{Q1}).

The following Lemma, whose the proof is postponed to \cref{appendix}, ensures that the two bandit problems are equivalent for high thresholds.

\begin{lemma}
	\label{lemma_threshold_u}
	\begin{equation}
	\label{threshold_u}
	\text{If}\quad u>\max\left(1, \left(\frac{2C'}{\min_{1\le k\le K}C_k}\right)^\frac{1}{\min_{1\le k\le K}\beta_k},
	\left(\frac{3\max_{1\le k\le K}C_k}{\min_{1\le k\le K}C_k}\right)^\frac{1}{\alpha_{(2)}-\alpha_{(1)}}\right)\,
	 ,
	\end{equation}
	then the unique best arm for the MAB problem is $\argmin_{1\le k\le
	K}\alpha_k=k^\ast$.
\end{lemma}

\begin{remark}
	\label{remarkMonitorU}
	Tuning the threshold $u$ based on the data is a difficult question, outside our scope.
	A standard practice is to monitor a relevant output (e.g. estimate of $\alpha$) as a function of the threshold $u$ and to pick the latter as low as possible in the stability region of the output.
	This is related to the Lepski's method, see e.g. \citet{boucheron2015}, \citet{carpentier2014adaptive}, \citet{hall1985}.
\end{remark}



\subsection{\textsc{Robust UCB} algorithm \citep{bubeck2013bandits}}
\label{guarantees_RUCB}
For the sake of completeness, we recall below the main feature of \textsc{Robust UCB}  and make explicit its theoretical guarantees in our setting. The  bound stated  in the following proposition is a direct consequence of the regret analysis conducted by \cite{bubeck2013bandits}.
\begin{proposition}
	\label{propTk}
	Applying the \textsc{Robust UCB} algorithm of~\citep{bubeck2013bandits}
	to our MAB problem, the expected number of
	times we pull any suboptimal arm $k\neq k^\ast$ is upper bounded as follows
	\begin{equation*}
	\mathbb{E}[T_{k, n}] = \mathcal{O}\left(\log n\right)\ .
	\end{equation*}
\end{proposition}

\begin{proof}
	See proof of Proposition 1 in \citet{bubeck2013bandits}.

\end{proof}

Hence, in expectation, \textsc{Robust UCB} pulls fewer times suboptimal arms than \textsc{ExtremeETC/Hunter}.
Indeed with \textsc{ExtremeETC/Hunter}, $T_{k, n}\ge
N=\Theta((\log n)^{2(2b+1)/b})$.

\begin{remark}
	\label{remarkRUCB_guarantees}
\cref{propTk} may be an indication that the Robust UCB approach performs better than \textsc{ExtremeETC/Hunter}. Nevertheless, guarantees on its expected extreme regret
require sharp concentration bounds on $T_{k, n}$ ($k\neq k^\ast$), which is out of the scope of this paper and left for future work.
\end{remark}

\begin{algorithm}[H]
	\caption{\textsc{Robust UCB} with truncated mean estimator
	\citep{bubeck2013bandits}}
	\begin{algorithmic}[1]
		\STATE {\bfseries Input:} $u>0$ s.t. \cref{threshold_u}, $\epsilon\in
		(0, 1]$ and $v>0$ s.t. \cref{eps_v}.
		\STATE {\bfseries Initialize:} Pull each arm once.
		\STATE {\bfseries for} $t\ge K+1$ {\bfseries do}
		\STATE \quad {\bfseries for} $k=1, \dots, K$ {\bfseries do}
		\STATE \qquad Update truncated mean estimator \\ $\widehat{\mu_k}
		\leftarrow \frac{1}{T_{k, t-1}}\sum_{s=1}^{t-1} Y_{k,
		s}\mathbbm{1}\left\{I_s=k, Y_{k,
		s}\le \bigl(\frac{v T_{k, s}}{\log(t^2)}\bigr)^\frac{1}{1+\epsilon}\right\}$\\
		\STATE \qquad Update index \\ $B_k \leftarrow \widehat{\mu_k} +
		4v^{1/(1+\epsilon)}\left(\frac{\log t^2}{T_{k, t-1}}\right)^{\epsilon/(1+\epsilon)}$
		\STATE \quad {\bfseries end for}
		\STATE \quad Play arm $I_t = \argmax_{1\le k\le K}B_k$
		\STATE {\bfseries end for}
	\end{algorithmic}
\end{algorithm}


\section{Numerical experiments}
\label{experiments}

In order to illustrate some aspects of the theoretical results presented previously,
	we consider a time horizon $n=10^5$ with $K=3$ arms and exact Pareto distributions with parameters given in \cref{param_exp}.
	Here, the optimal arm is the second one (incidentally, the distribution with highest mean is the first one).

\begin{table}
\begin{center}
\begin{tabular}{l | l l l}
\toprule
\ & \multicolumn{3}{c}{Arms}\\
\  & \textbf{$\quad k=1$} & \textcolor{red}{\textbf{$\quad k^\ast=2$}} & \textbf{$\quad k=3$} \\
\midrule
$\alpha_k$ & $\quad 15$ & $\quad 1.5$ & $\quad 10$ \\
$C_k$ & $\quad 10^8$ & $\quad 1$ & $\quad 10^5$ \\
$\mathbb{E}\left[X_{k, 1}\right]$ & $\quad 3.7$ & $\quad 3$ & $\quad 3.5$ \\
$\mathbb{E}\left[\max_{1\le t\le n}X_{k, t}\right]$ & $\quad 7.7$ & $\quad \textcolor{red}{5.8\cdot 10^3}$ & $\quad 11$ \\
\bottomrule
\end{tabular}
\vspace{.1in}
\caption{Pareto distributions used in the experiments.}\label{param_exp}
\end{center}
\end{table}

\begin{figure}
\begin{subfigure}{.5\textwidth}
  \centering
  \includegraphics[height=.7\linewidth]{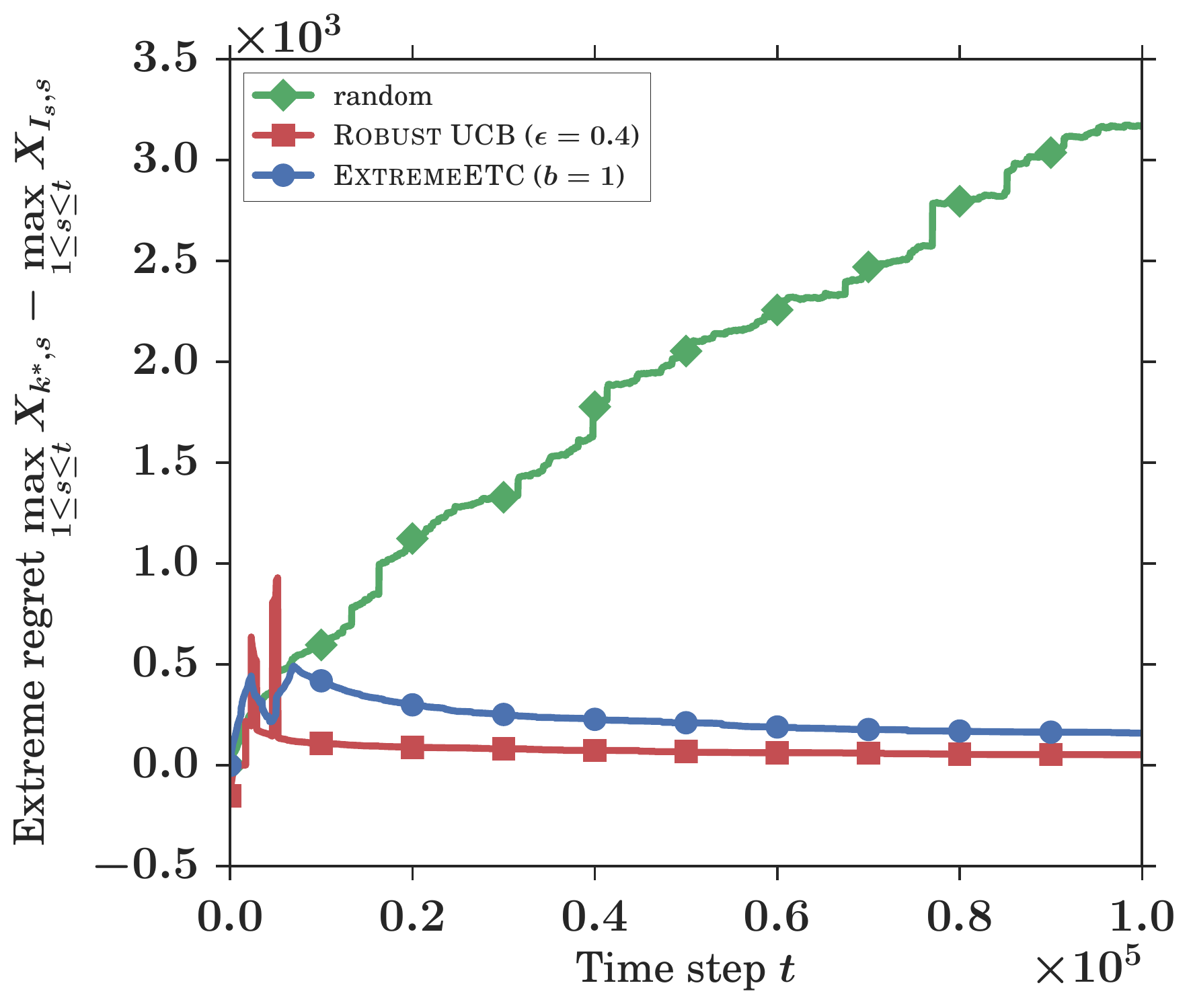}
  \caption{}
  \label{simu_normal}
\end{subfigure}%
\begin{subfigure}{.5\textwidth}
  \centering
  \includegraphics[height=.7\linewidth]{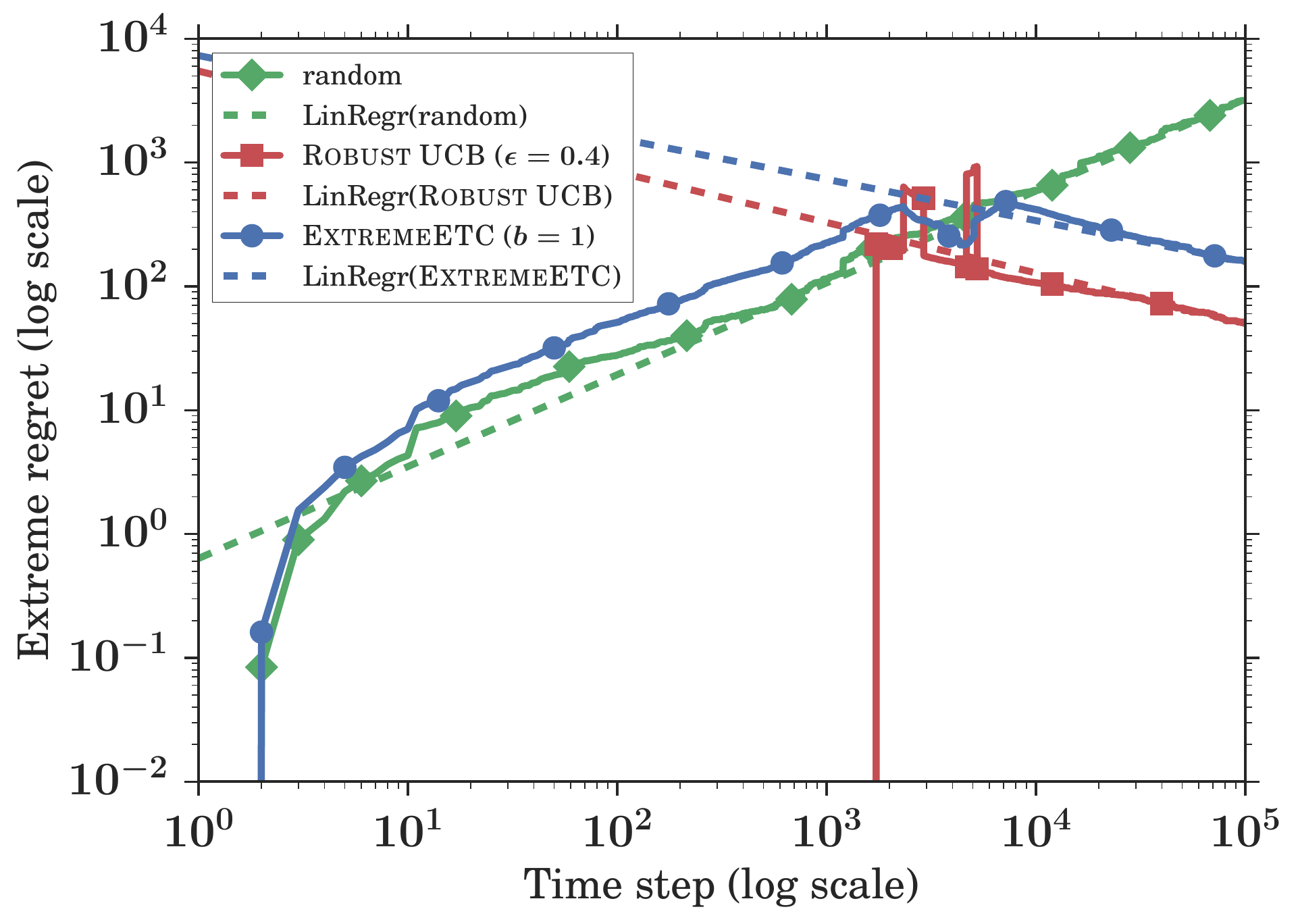}
  \caption{}
  \label{simu_loglog}
\end{subfigure}
\caption{Averaged extreme regret (over $1000$ independent simulations)
for \textsc{ExtremeETC}, \textsc{Robust UCB} and a uniformly random strategy. \cref{simu_loglog} is the log-log scaled counterpart of \cref{simu_normal} with linear regressions computed over $t=5\cdot 10^4, \dots, 10^5$.}
\label{simu}
\end{figure}

We have implemented \textsc{Robust UCB} with parameters $\epsilon=0.4$, which satisfies $1+\epsilon<\alpha_2=1.5$, $v$ achieving the equality in \cref{eps_v} (ideal case) and a threshold $u$ equal to the lower bound in \cref{threshold_u} plus $1$ to respect the strict inequality.
\textsc{ExtremeETC} is runned with $b=1<+\infty=\min_{1\le k\le K}\beta_k$.
In this setting, the most restrictive condition on the time horizon, $n > KN \approx 7000$ (given by \cref{N}), is checked, which places us in the validity framework of \textsc{ExtremeETC}.
The resulting strategies are compared to each other and to the random strategy pulling each arm uniformly at random, but not to \textsc{Threshold Ascent} algorithm \citep{streeter2006simple} which is designed only for bounded rewards.
Precisely, $1000$ simulations have been run and Figure~\ref{simu} depicts the extreme regret~\eqref{eq:def_regret} in each setting averaged over these $1000$ trajectories.
These experiments empirically support the theoretical bounds in \cref{upper}: the expected extreme regret of \textsc{ExtremeETC} converges to zero for large horizons.
On the log-log scale (\cref{simu_loglog}), \textsc{ExtremeETC}'s extreme regret starts linearly decreasing after the initialization phase, at $n > KN\approx 7000$,
which is consistent with \cref{lemTast}.
The corresponding linear regression reveals a slope $\approx -0.333$ (with a coefficient of determination $R^2\approx 0.97$),
which confirms \cref{theo:regretExtremeHunter} and \cref{thm_lower} yielding the theoretical slope $-(1-1/\alpha_{k^\ast})=-1/3$.


\section{Conclusion}

This paper brings two main contributions. It first provides
a refined regret bound analysis of the performance of the \textsc{ExtremeHunter} algorithm in the context of the max $K$-armed
bandit problem that significantly improves upon the results obtained in the seminal
contribution \cite{carpentier2014extreme}, also proved to be valid for \textsc{ExtremeETC}, a computationally appealing alternative we introduce.
In particular, the obtained upper
bound on the regret converges to zero for large horizons and is shown
to be tight when the tail of the rewards is sufficiently close to a
Pareto tail (second order parameter $b\ge 1$).  On the other hand, this paper
offers a novel view of this approach, interpreted here as a specific
version of a classical solution (\textit{Robust UCB}) of the MAB
problem, in the situation when only very large rewards matter.

Based on these encouraging results, several lines of further research
can be sketched. In particular, future work will investigate to which
extent the lower bound established for \textsc{ExtremeETC/Hunter}
holds true for any strategy with exploration
stage of the same duration, and whether  improved  performance is
achievable with alternative stopping
criteria for the exploration stage.

\section*{Acknowledgments}
This work was supported by a public grant (\textit{Investissement d'avenir} project, reference ANR-11-LABX-0056-LMH,
LabEx LMH) and by the industrial chair \textit{Machine Learning for Big Data} from Télécom ParisTech.

\bibliographystyle{apalike}
\bibliography{bib}

\newpage

\appendix
\onecolumn

\section{Appendix}
\label{appendix}

\subsection{Proof of \cref{lem:Q5}}

\begin{proof}
  For $n>Q_3$ (defined in \cref{def:Q3}),
  one has $V_{k^\ast}>2\max_{k\neq k^\ast}V_k$, which implies that $\max_{k\neq k^\ast}V_k/(V_{k^\ast}-V_k) < 1$.
  Hence
  \[
  \max_{k\neq k^\ast}e^{\big(F\sqrt{1+\rho}A_0^{-b/(2b+1)}\frac{V_k}{V_{k^\ast}-V_k}\big)^2}<e^{\left(F\sqrt{1+\rho}A_0^{-b/(2b+1)}\right)^2}\le Q_5\ .
  \]
  Then, as $T_{k, t}\ge N$ and by definitions of $N$ (\cref{N}) and $\delta_0$ (\cref{delta0}), we have for $n>Q_5$
  that for any suboptimal arm $k\neq k^\ast$
  \begin{equation*}
    \begin{split}
    &(C_k n)^{1/\alpha_k}\Gamma(1-1/\alpha_k)
    \left(1+F\log n \sqrt{\log(n/\delta_0)}T_{k, t}^{-b/(2b+1)}\right)\\
    &<(C_ {k^\ast} n)^{1/\alpha_{k^\ast}}\Gamma(1-1/\alpha_{k^\ast})\, ,
  \end{split}
  \end{equation*}
which implies, using \cref{boundsBk}, that under $\xi_1$: $B_{k, t}<B_{k^\ast, t}$ for $t>KN$.
\end{proof}

\subsection{Proof of \cref{upper}}

\begin{proof}
	We want to upper bound
	$\mathbb{E}[R_n]=\mathbb{E}[G_n^{(k^\ast)}]-\mathbb{E}[G_n^{(\pi)}]$.
	To do so, we lower bound $\mathbb{E}[G_n^{(\pi)}]$ as follows
	\begin{equation*}
		\mathbb{E}\left[G_n^{(\pi)}\right]
		=\mathbb{E}\left[\max_{t\le n}X_{I_t, t}\right]
		\ge \mathbb{E}\left[\max_{\{t\le n, I_t=k^\ast\}}X_{I_t, t}\right]=\mathbb{E}\left[\max_{\{i\le T_{k^\ast, n}\}}\widetilde{X}_{k^\ast, i}\right].
	\end{equation*}
	Thus
	\begin{equation*}
		\mathbb{E}\left[G_n^{(\pi)}\right]
		\ge \mathbb{E}\left[\max_{i\le n-(K-1)N}\widetilde{X}_{k^\ast, i}\mathbbm{1}\{\xi_1\}\right]\, ,
	\end{equation*}
	where we used that under $\xi_1$, $T_{k^\ast, n}=n-(K-1)N$.
	Now we call the following result (\cref{lemma_upper}, proved in \cref{subsec:proof_lemma_upper}), giving a lower bound on the expected maximum of i.i.d. second order Pareto r.v. given some event.

	\begin{lemma}
	\label{lemma_upper}
	Let $X_1, ..., X_T$ be i.i.d. samples from an $(\alpha, \beta, C, C')$-second
	order Pareto distribution.
	Let $\xi$ be an event of probability larger than $1-\delta$.
	If $\delta<1/2$ and $T\ge\max\left(4c, (4c)^{1/\beta}
	\log(2)C(2C')^{1/\beta}, 8\log^2(2)\right)$ for a given constant $c$ depending
	only on $\beta, C$ and $C'$, we have
	\begin{equation*}
	\begin{split}
	&\mathbb{E}\left[\max_{1\le i\le T}X_i\mathbbm{1}\{\xi\}\right]
	\ge (TC)^{1/\alpha}\Gamma\left(1-\frac{1}{\alpha}\right)
	-\left(4+\frac{8}{\alpha-1}\right)(TC)^{1/\alpha}\delta'^{1-1/\alpha}\\
	&-2\left(\frac{4D_2C^{1/\alpha}}{T^{1-1/\alpha}}+\frac{2C'D_{\beta+1}}{C^{\beta+1-1/\alpha}T^{\beta-1/\alpha}}
	+2(2C'T)^{1/(\alpha(1+\beta))}e^{-HT^{\beta/(\beta+1)}}\right).
	\end{split}
	\end{equation*}
	\end{lemma}

	Then, applying \cref{lemma_upper} with $\xi=\xi_1$ and $\delta=\delta_0$ we obtain after simplification
	\begin{equation*}
		\begin{split}
		\mathbb{E}[R_n]
		\le H'n^{1/\alpha_{k^\ast}}\Bigl\{&\frac{1}{n}+\frac{1}{n^b}+\frac{K}{n}(\log n)^{2(2b+1)/b}
		+\delta_0^{1-1/\alpha_{k^\ast}}\\
		&+n^{1/(\alpha_{k^\ast}(1+\beta_{k^\ast}))}e^{-H_{k^\ast}(n/2)^{\beta/(\beta+1)}}\Bigr\}\, ,
	\end{split}
	\end{equation*}
	where $H_{k^\ast}=\frac{1}{2}C_{k^\ast}(2C')^{1/(\alpha_{k^\ast}(1+\beta_{k^\ast}))}$ and
	$H'$ is a constant depending only on $(\alpha_k, \beta_k, C_k)_{1\le k\le
	K}$ and C'.
	The definition of $\delta_0$ concludes the proof.

\end{proof}

\subsection{Proof of \cref{lemma_upper}}
\label{subsec:proof_lemma_upper}

\begin{proof}

We follow the proof of Lemma 2 in \citet{carpentier2014extreme} except that we use
\cref{thm1} instead of their Theorem 1.
Let $x_{\delta}$ be such that $\mathbb{P}(\max_{1\le t\le T}X_t\le
x_{\delta})=1-\delta$.
Then we have
\begin{equation*}
\begin{split}
&\mathbb{E}\left[\max_{1\le t\le T}X_t\mathbbm{1}\{\xi\}\right]
= \mathbb{E}\left[\max_{1\le t\le T}X_t\right]-\mathbb{E}\left[\max_{1\le t\le
T}X_t\mathbbm{1}\{\bar{\xi} \}\right]\\
& = \mathbb{E}\left[\max_{1\le t\le T}X_t\right] -
\int_0^{x_{\delta}}\mathbb{P}\left(\max_{1\le t\le
T}X_t\mathbbm{1}\{\bar{\xi} \}>x\right)  \, \mathrm{d}x\\
&\qquad\qquad\qquad\qquad- \int_{x_{\delta}}^\infty\mathbb{P}\left(\max_{1\le t\le
T}X_t\mathbbm{1}\{\bar{\xi} \}>x\right)  \, \mathrm{d}x\\
& \ge \mathbb{E}\left[\max_{1\le t\le T}X_t\right] - \delta x_{\delta}
- \int_{x_{\delta}}^\infty\mathbb{P}\left(\max_{1\le t\le
T}X_t\mathbbm{1}\{\bar{\xi} \}>x\right)  \, \mathrm{d}x\, ,
\end{split}
\end{equation*}
where the inequality comes from $\mathbb{P}\left(\max_{1\le t\le T}X_t\mathbbm{1}\{\bar{\xi} \}>x\right)\le
\mathbb{P}\left(\bar{\xi} \right)\le \delta$.
Since $T\ge \log(2)\max\left(C(2C')^{1/\beta}, 8\log(2)\right)$ and $\delta<1/2$, we have
from Lemma 3 in \citet{carpentier2014extreme}
\begin{equation*}
\begin{split}
&\left|\mathbb{P}\left(\max_{1\le i\le T}X_i \le (TC/\log(1/(1-\delta)))^{1/\alpha}\right)-(1-\delta)\right|\\
&\le (1-\delta)\left(\frac{4}{T}\left(\log\frac{1}{1-\delta}\right)^2
+\frac{2C'}{C^{1+\beta}}\left(\log\frac{1}{1-\delta}\right)^{1+\beta}\right)\\
&\le \frac{4}{T}(2\delta)^2+\frac{2C'}{C^{1+\beta}}(2\delta)^{1+\beta}
\le c\delta\max\left(\frac{\delta}{T}, \frac{\delta^\beta}{T^\beta}\right)
\le c\delta\max\left(\frac{1}{T}, \frac{1}{T^\beta}\right)\, ,
\end{split}
\end{equation*}
where $c$ is a constant that depends only on $C, C'$ and $\beta$.
As we have $c\max(T^{-1}, T^{-\beta})\le 1/4$, this implies
\begin{equation*}
x_{-}=(TC/\log(1/(1-2\delta)))^{1/\alpha}
\le x_{\delta}
\le (TC/\log(1/(1-\delta/2)))^{1/\alpha}= x_{+}\ .
\end{equation*}
It follows
\begin{equation*}
\mathbb{E}\left[\max_{1\le t\le T}X_t\mathbbm{1}\{\xi\}\right]
\ge \mathbb{E}\left[\max_{1\le t\le T}X_t\right] - \delta x_{+} -
\int_{x_{-}}^\infty\mathbb{P}\left(\max_{1\le t\le T}X_t>x\right)  \, \mathrm{d}x\ .
\end{equation*}
From \cref{thm1} we deduce
\begin{equation*}
\begin{split}
&\mathbb{E}\left[\max_{1\le t\le T}X_t\mathbbm{1}\{\xi\}\right]
\ge \mathbb{E}\left[\max_{1\le t\le T}X_t\right] - \delta x_{+}
- \int_{x_{-}}^\infty (1-e^{-TCx^{-\alpha}})  \, \mathrm{d}x\\
&-\left(\frac{4D_2C^{1/\alpha}}{T^{1-1/\alpha}}+\frac{2C'D_{\beta+1}}{C^{\beta+1-1/\alpha}T^{\beta-1/\alpha}}
+2(2C'T)^{1/(\alpha(1+\beta))}e^{-HT^{\beta/(\beta+1)}}\right).
\end{split}
\end{equation*}
From the proof of Lemma 2 in \citet{carpentier2014extreme} we have for $\delta$ small
enough
\begin{equation*}
\int_{x_{-}}^\infty (1-e^{-TCx^{-\alpha}})  \, \mathrm{d}x
\le \frac{8}{\alpha-1}(TC)^{1/\alpha}\delta^{1-1/\alpha}
\end{equation*}
and
\begin{equation*}
\delta x_{+}\le 4(TC)^{1/\alpha}\delta^{1-1/\alpha}\ .
\end{equation*}
\cref{thm1} concludes the proof.

\end{proof}

\subsection{Proof of \cref{power}}

\begin{proof}

	Let $F$ and $F_r$ be respectively the cumulative distribution functions of
	$X$ and $X^r$. For $x\ge 0$,
	\begin{equation*}
	F_r(x)=\mathbb{P}(X^r\le x) = \mathbb{P}(X\le x^{1/r}) = F(x^{1/r})\ .
	\end{equation*}
	As $X$ follows an $(\alpha, \beta, C, C')$-second order Pareto distribution
	we have
	\begin{equation*}
	|1-Cx^{-\alpha/r}-F_r(x)|
	=|1-Cx^{-\alpha/r}-F(x^{1/r})|
	\le C'x^{-(\alpha/r)(1+\beta)}\, ,
	\end{equation*}
	which concludes the proof.

\end{proof}

\subsection{Proof of \cref{xi12}}

\begin{proof}
	We first state the following result (\cref{high_prob_bounds}, proved in \cref{proof_high_prob_bounds}), yielding high probability lower and upper bounds for the maximum of i.i.d. second order Pareto r.v.

	\begin{lemma}
		\label{high_prob_bounds}
		For $X_1, ..., X_T$ i.i.d samples drawn from an
		$(\alpha,\beta,C,C')$-second-order Pareto distribution we define high
		probability lower and upper bound
		\begin{equation*}
			\ell(T,\delta)= \left(\frac{TC}{2\log\frac{1}{\delta}}\right)^{1/\alpha}\quad \text{and}\quad L(T,\delta)= \left(\frac{4TC}{\log\frac{1}{1-\delta}}\right)^{1/\alpha}\, ,
		\end{equation*}
		where $\delta\in(0,1)$ can depend on $T$ and is such that
		$\lim_{T\rightarrow\infty}\ell(T,\delta)=\infty$ and
		$\lim_{T\rightarrow\infty}L(T,\delta)=\infty$.
		For $T$ large enough such that $C\ell(T,\delta)^{-\alpha}\ge
		2C'\ell(T,\delta)^{-\alpha(1+\beta)}$, $CL(T,\delta)^{-\alpha}\ge
		C'L(T,\delta)^{-\alpha(1+\beta)}$ and $L(T,\delta)^{-\alpha}\le \frac{1}{4C}$
		we have
		\begin{equation}
			\label{proba_lL}
			\mathbb{P}\left(\max_{1\le i\le T}X_i\le \ell(T,\delta)\right)\le\delta \quad \text{and}\quad \mathbb{P}\left(\max_{1\le i\le T}X_i\ge L(T,\delta)\right)\le\delta.
		\end{equation}
	\end{lemma}

	With the notations of \cref{high_prob_bounds}, we respectively denote by $\ell_k$ and $L_k$ the high probability lower and upper bounds for any arm $k$.
	Using \cref{proba_lL} we have by a union bound that
	with probability higher than $1-K\delta_0$
	\begin{equation*}
		\max_{1\le i\le n-(K-1)N}\widetilde{X}_{k^\ast,i}\ge \ell_{k^\ast}(n-(K-1)N,\
		\delta_0)\, ,
	\end{equation*}
	and for any suboptimal arm $k\neq k^\ast$
	\begin{equation*}
		\max_{1\le i\le N}\widetilde{X}_{k,i}\le L_k(N,\ \delta_0)\ .
	\end{equation*}
	Under this event, using the definition of the confidence level $\delta_0$ we observe for $n$ larger than some constant that
	for any suboptimal arm $k\neq k^\ast$, $L_k(N,\delta_0)\le\ell_{k^\ast}(n-(K-1)N,\delta_0)$, which concludes the proof.

\subsection{Proof of \cref{high_prob_bounds}}
\label{proof_high_prob_bounds}

\begin{proof}

For the high probability lower bound we write:
\begin{equation*}
\begin{split}
&\mathbb{P}\left(\max_{1\le i\le T}X_i\le \ell(T,\delta)\right)
=\mathbb{P}(X_1\le \ell(T,\delta))^T\\
&\le \left(1-C\ell(T,\delta)^{-\alpha}+C'\ell(T,\delta)^{-\alpha(1+\beta)}\right)^T\\
&\le \left(1-\frac{1}{2}TC\ell(T,\delta)^{-\alpha}\right)^T
\le e^{-\frac{1}{2}TC\ell(T,\delta)^{-\alpha}}
=\delta.
\end{split}
\end{equation*}
And for the high probability upper bound:
\begin{equation*}
\begin{split}
&\mathbb{P}\left(\max_{1\le i\le T}X_i\le L(T,\delta)\right)
=\mathbb{P}(X_1\le L(T,\delta))^T\\
&\ge \left(1-CL(T,\delta)^{-\alpha}-C'L(T,\delta)^{-\alpha(1+\beta)}\right)^T\\
&\ge (1-2CL(T,\delta)^{-\alpha})^T
\ge e^{-4TCL(T,\delta)^{-\alpha}}
=1-\delta\ .
\end{split}
\end{equation*}

\end{proof}

  \subsection{Proof of \cref{lemma_threshold_u}}

	From \cref{thm1}, we have for any arm $k\in\{1,\; \ldots,\; K  \}$,
	\begin{equation*}
	\begin{split}
	\mathbb{E}\left[X_{k, 1}\mathbbm{1}\{X_{k, 1}>u\}\right]
	& \le \int^\infty_0 \mathbb{P}(X_{k, 1}\mathbbm{1}\{X_{k, 1}>u\}\ge x) \,
	\mathrm{d}x\\
	&= u(1-F_k(u)) + \int^\infty_u (1-F_k(x)) \, \mathrm{d}x
	\le M_k + \Delta_k\, ,
	\end{split}
	\end{equation*}
	where $M_k=(C_k\alpha_k/(\alpha_k-1))u^{-\alpha_k + 1}$ and
	$\Delta_k=(C'\alpha_k(1+\beta_k)/(\alpha_k(1+\beta_k)-1))u^{-\alpha_k(1+\beta_k)+1}$.
	Similarly, we have $\mathbb{E}[X_{k, 1}\mathbbm{1}\{X_{k,
	1}>u\}] \ge~M_k-\Delta_k$.\\
	For $u$ large enough, we want to prove that
	$M_{k^\ast}-\Delta_{k^\ast}>M_k+\Delta_k$ for any arm $k\neq k^\ast$, which would
	prove that $\argmax_{1\le k\le K}\mathbb{E}[Y_{k, 1}]=k^\ast$.
	First, we observe for $u>\max(1, (2C'/\min_{1\le k\le K}C_k)^{1/\min_{1\le k\le K}\beta_k})$
	that $\Delta_k < \frac{1}{2}M_k$.
	Then, for
	$u>(3\max_{1\le k\le K}C_k/\min_{1\le k\le K}C_k)^{1/(\alpha_{(2)}-\alpha_{(1)})}$, we have that
	$\frac{1}{2}M_{k^\ast}>\frac{3}{2}M_k$ for any arm $k\neq k^\ast$, which concludes the proof.

\end{proof}

\end{document}